\documentclass[final,1p,times]{elsarticle}

\usepackage{caption,graphicx,url,amsmath,amssymb}
\usepackage{subfigure,underoverlap}
\usepackage{booktabs}
\usepackage[raggedright]{sidecap}
\sidecaptionvpos{figure}{c} 

\usepackage{setspace}
\usepackage{color}

\usepackage{epstopdf}
\usepackage[utf8]{inputenc}

\usepackage{amsthm}
\newtheorem{mydef}{Definition}

\newtheorem{lemma}{Lemma}
\newtheorem{theorem}{Theorem}

\usepackage{algorithm}
\usepackage{algpseudocode}
\algrenewcommand\algorithmicrequire{\textbf{Input}}
\algrenewcommand\algorithmicensure{\textbf{Output}}

\def\1{{\mathbf 1}}

\def\0{{\mathbf 0 }}

\def\beq{\begin{equation} }
\def\eeq{\end{equation} }

\usepackage[dvipsnames]{xcolor}

\journal{Pattern Recognition}

\begin{document}

\begin{frontmatter}

\title{Time Series Cluster Kernel for Learning Similarities between Multivariate Time Series with Missing Data}

\author[UiT,ML]{Karl Øyvind Mikalsen\corref{cor1} }
\address[UiT]{Dept. of Math. and Statistics, UiT The Arctic University of Norway, Tromsø, Norway}
\address[ML]{UiT Machine Learning Group}

\cortext[cor1]{Corresponding author at: Department of Mathematics and Statistics, Faculty of Science and Technology,
UiT – The  Arctic University of Norway,
N-9037 Tromsø, Norway}

\ead{karl.o.mikalsen@uit.no}

\author[ML,ML_IFT]{Filippo Maria Bianchi}
\address[ML_IFT]{Dept. of Physics and Technology, UiT, Tromsø, Norway}

\author[ML,spain]{Cristina Soguero-Ruiz}
\address[spain]{Dept. of Signal Theory and Comm., Telematics and Computing, Universidad Rey Juan Carlos, Fuenlabrada, Spain}

\author[ML,ML_IFT]{Robert Jenssen}

\begin{abstract}

Similarity-based approaches represent a promising direction for time series analysis. However, many such methods rely on parameter tuning, and some have shortcomings if the time series are multivariate (MTS), due to dependencies between attributes, or the time series contain missing data. In this paper, we address these challenges within the powerful context of kernel methods by proposing the robust \emph{time series cluster kernel} (TCK). The approach taken leverages the missing data handling properties of Gaussian mixture models (GMM) augmented with informative prior distributions. 
An ensemble learning approach is exploited to ensure robustness to parameters by combining the clustering results of many GMM to form the final kernel.

We evaluate the TCK on synthetic and real data and compare to other state-of-the-art techniques. The experimental results demonstrate that the TCK is robust to parameter choices, provides competitive results for MTS without missing data and outstanding results for missing data. 

\end{abstract}

\begin{keyword}
Multivariate time series \sep Similarity measures \sep Kernel methods \sep Missing data \sep Gaussian mixture models \sep Ensemble learning
\end{keyword}
\end{frontmatter}

\section{Introduction}

Time series analysis is an important and mature research topic, especially in the context of univariate time series (UTS) prediction \cite{BoxJenkins,Chatfield,CryerChan,shumway}. The field tackles real world problems in many different areas such as energy consumption~\cite{iglesias2013analysis}, climate studies~\cite{ji2013dynamic}, biology~\cite{10.1371/journal.pone.0084955}, medicine~\cite{Hayrinen2008,soguero2015data, Soguero2016predicting} and finance~\cite{hsu2014clustering}. However, the need for analysis of multivariate time series (MTS)~\cite{tsay2013multivariate} is growing in modern society as data is increasingly collected simultaneously from multiple sources over time, often plagued by severe missing data problems~\cite{AnovaHazanZeevi,BashirWei}. These challenges complicate analysis considerably, and represent open directions in time series analysis research. The purpose of this paper is to answer such challenges, which will be achieved within the context of the powerful \emph{kernel methods}~\cite{scholkopf2001learning,shawe2004kernel} for reasons that will be discussed below.   

Time series analysis approaches can be broadly categorized into two families: (i) \textit{representation methods}, which provide high-level features for representing properties of the time series at hand, and (ii) \textit{similarity measures}, which yield a meaningful similarity between different time series for further analysis~\cite{Wang:2013, Aghabozorgi201516}. 

Classic representation methods are for instance Fourier transforms, wavelets, singular value decomposition, symbolic aggregate approximation, and piecewise aggregate approximation, \cite{Faloutsos:1994:FSM:191843.191925, chan1999efficient, korn1997efficiently, lin2007experiencing, keogh2001dimensionality}. 
Time series may also be represented through the parameters of model-based methods such as Gaussian mixture models (GMM)~\cite{Marlin:2012:UPD:2110363.2110408, bashir2005automatic, bashir2007object}, Markov models and hidden Markov models (HMMs)~\cite{Ramoni:2002:BCD:584647.584652, panuccio2002hidden, knab2003model}, time series bitmaps~\cite{kumar2005time} and variants of ARIMA~\cite{Corduas20081860, xiong2002mixtures}.  An advantage with parametric models is that they can be naturally extended to the multivariate case. 
For detailed overviews on representation methods, we refer the interested reader to \cite{Wang:2013,
Aghabozorgi201516, fu2011review}.  

Of particular interest to this paper are similarity-based approaches. Once defined, such similarities between pairs of time series may be utilized in a wide range of applications, such as classification, clustering, and anomaly detection~\cite{han2011data}.    
Time series similarity measures include for example dynamic time warping (DTW) \cite{Berndt:1994:UDT:3000850.3000887}, the longest common subsequence (LCSS)~\cite{vlachos2003indexing}, the extended Frobenius norm (Eros)~\cite{yang2007efficient}, and the Edit Distance with Real
sequences (EDR)~\cite{Chen:2005:RFS:1066157.1066213}, and represent state-of-the-art performance in UTS prediction~\cite{Wang:2013}. 
However, many of these measures cannot straightforwardly be extended to MTS such that they take relations between different attributes into account~\cite{banko2012correlation}. 
The  learned pattern similarity (LPS) is  an exception, based on the identification of segments-occurrence within the time series, which generalizes naturally to MTS~\cite{baydogan2016time} by means of regression trees where a bag-of-words type compressed representation is created, which in turn is used to compute the similarity.

A similarity measure that also is positive semi-definite (psd) is a \textit{kernel}~\cite{shawe2004kernel}. Kernel methods~\cite{shawe2004kernel,Jenssen2010,Jenssen2013} have dominated machine learning and pattern recognition over two decades and have been very successful in many fields~\cite{scholkopf2004kernelC,camps2009kernel,SogueroJBHI}. A main reason for this success is the well understood theory behind such methods, wherein nonlinear data structures can be handled via an implicit or explicit mapping to a reproducing kernel Hilbert space (RKHS)~\cite{scholkopf2001generalized, berlinet2011} defined by the choice of kernel. Prominent examples of kernel methods include the support vector machine (SVM)~\cite{steinwart2008support} and kernel principal component analysis (kPCA)~\cite{scholkopf1997kernel}.

However, many similarities (or equivalently dissimilarities) are non-metric as they do not satisfy the triangle-inequality, and in addition most of them are not psd and therefore not suited for kernel methods~\cite{haasdonk2004learning, marteau2015recursive}.
Attempts have been made to design kernels from non-metric distances such as DTW, of which the global alignment kernel (GAK) is an example \cite{cuturi2011fast}.
There are also promising works on deriving kernels from parametric models, such as the probability product kernel~\cite{jebara2004probability}, Fisher kernel~\cite{jaakkola1999using}, and reservoir based kernels~\cite{chen2013model}. Common to all these methods is however a strong dependence on a correct hyperparameter tuning, which is difficult to obtain in an unsupervised setting. 
Moreover, many of these methods cannot naturally be extended to deal with MTS, as they only capture the similarities between individual attributes and do not model the dependencies between multiple attributes~\cite{banko2012correlation}. 
Equally important, these methods are not designed to handle missing data, an important limitation in many existing scenarios, such as clinical data where MTS originating from Electronic Health Records (EHRs) often contain missing data~\cite{Hayrinen2008,soguero2015data, Soguero2016predicting,liu2016learning}. 

In this work, we propose a new kernel for computing similarities between MTS that is able to handle missing data without having to resort to imputation methods~\cite{donders2006review}. We denote this new measure as the \textit{time series cluster kernel} (TCK). Importantly, the novel kernel is robust and designed in an unsupervised manner, in the sense that no critical hyperparameter choices have to be made by the user. The approach taken is to leverage the missing data handling properties of GMM modeling following the idea of~\cite{Marlin:2012:UPD:2110363.2110408}, where robustness to sparsely sampled data is ensured by extending the GMM using informative prior distributions. 
However, we are not fitting a single parametric model, but rather exploiting an ensemble learning approach~\cite{Dietterich2000} wherein robustness to hyperparameters is ensured by joining the clustering results of many GMM to form the final kernel. This is to some degree analogous to the approaches taken in \cite{cuturi2011autoregressive} and \cite{IzquierdoVerdiguier20151299}. 
More specifically, each GMM is initialized with different numbers of mixture components and random initial conditions and is fit to a randomly chosen subsample of the data, attributes and time segment, through an embarrassingly parallel procedure. 
This also increases the robustness against noise.
The posterior assignments provided by each model are combined to form a kernel matrix, i.e.\ a psd similarity matrix. This opens the door to clustering, classification, etc., of MTS within the framework of kernel methods, benefiting from the vast body of work in that field. 
The procedure is summarized in Fig. \ref{fig:general_procedure}.

\begin{figure}[!t]
  \centering
  \includegraphics[trim = 0mm 1mm 0mm 1mm, clip, width=\columnwidth, keepaspectratio]{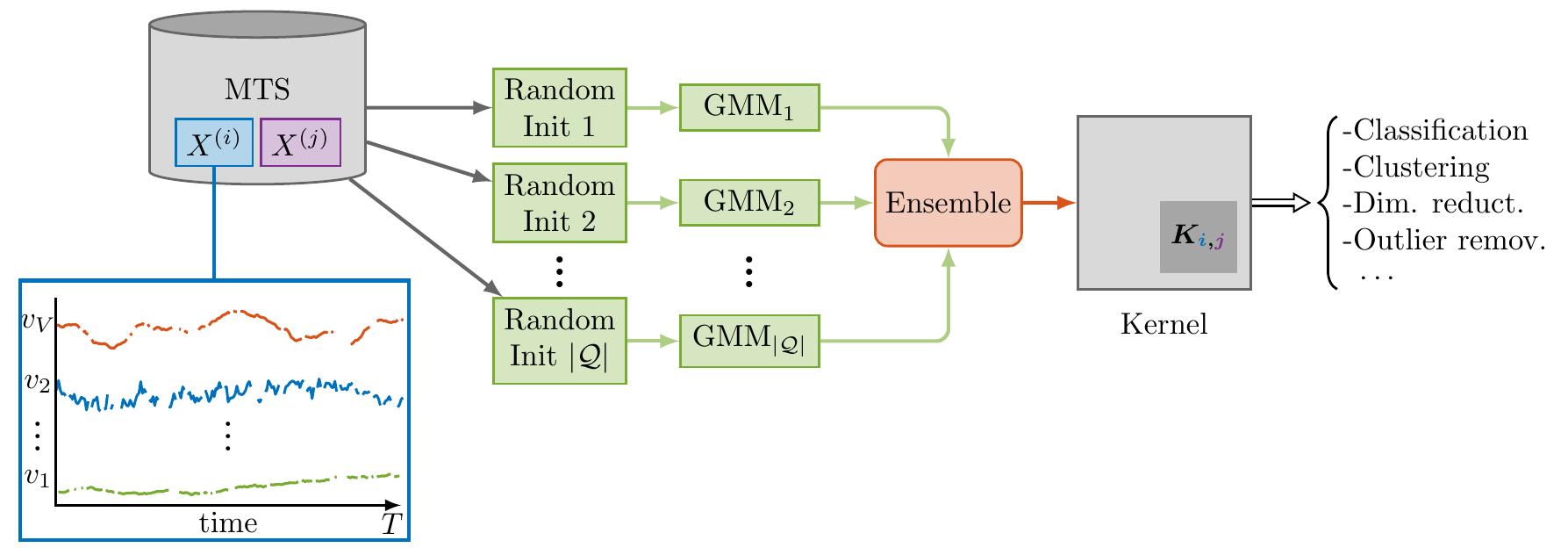}
  \caption{Schematic depiction of the procedure used to compute the TCK.}
  \label{fig:general_procedure}
\end{figure}

In the experimental section we illustrate some of the potentials of the TCK by applying it to classification, clustering, dimensionality reduction and visualization tasks. 
In addition to the widely used DTW, we compare to GAK and LPS. The latter  inherits the decision tree approach to handle missing data, is similar in spirit to the TCK in the sense of being based on an ensemble strategy~\cite{baydogan2016time}, and is considered the state-of-the-art for MTS.  As an additional contribution, we show  in \ref{app:lps} that the LPS is in fact a kernel itself, a result that to the authors best knowledge has not been proven before. The experimental results demonstrate that TCK is very robust to hyperparameter choices, provides competitive results for MTS without missing data and outstanding results for MTS with missing data. This we believe provides a useful tool across a variety of applied domains in MTS analysis, where missing data may be problematic.

The remainder of the paper is organized as follows. In Section~\ref{sec:related_works} we present related works, whereas in Section~\ref{sec:Background}, we give the background needed for building the proposed method. In Section~\ref{sec:TCK} we provide the details of the TCK, whereas in Section~\ref{sec:experiments} we evaluate it on synthetic and real data and compare to LPS and DTW. Section~\ref{sec:conclusions} contains conclusions and future work.

\section{Related work}
\label{sec:related_works}

While several (dis)similarity measures have been defined over the years to compare time series, many of those measures are not psd and hence not suitable for kernel approaches.
In this section we review some of the main kernels functions that have been proposed for time series data.

The simplest possible approach is to treat the time series as vectors and apply well-known kernels such as a linear or radial basis kernel~\cite{scholkopf2001learning}. 
While this approach works well in some circumstances, time dependencies and the relationships among multiple attributes in the MTS are not explicitly modeled.

DTW~\cite{Berndt:1994:UDT:3000850.3000887} is one of the most commonly used similarity measures for UTS and has become the state-of-the-art in many practical applications~\cite{Cai2015181,ratanamahatana2004everything,Lines}. 
Several formulations have been proposed to extend DTW to the multidimensional setting \cite{shokoohi2015generalizing, banko2012correlation}.
Since DTW does not satisfy the triangle inequality, it is not negative definite and, therefore, one cannot obtain a psd kernel by applying an exponential function to it~\cite{berg1984harmonic}. 
Such an indefinite kernel may lead to a non-convex optimization problem (e.g., in an SVM), which hinders the applicability of the model~\cite{haasdonk2004learning}. 
Several approaches have been proposed to limit this drawback at the cost of more complex and costly computations.
In \cite{wu2005learning, Chen:2009:SCC:1577069.1577096} ad hoc spectral transformations were employed to obtain a psd matrix. 
Cuturi et al.~\cite{cuturi2011fast} designed a DTW-based kernel using global alignments (GAK). 
Marteau and Gibet proposed an approach that combines DTW and edit distances with a recursive regularizing term~\cite{ marteau2015recursive}.  

Conversely, there exists a class of (probabilistic) kernels operating on  the configurations of a given parametric model, where the idea is to leverage the way distributions capture similarity.
For instance, the Fisher kernel assumes an underlying generative model to explain all observed data~\cite{jaakkola1999using}. 
The Fisher kernel maps each time series $x$ into a feature vector $U_x$, which is the gradient of the log-likelihood of the generative model fit on the dataset. 
The kernel is defined as $K(x_i, x_j) = U_{x_i}^T\mathcal{I}^{-1}U_{x_j}$, where $\mathcal{I}$ is the fisher information matrix.
Another example is the probability product kernel~\cite{jebara2004probability}, which is evaluated by means of the Bhattacharyya distance in the probability space.
A further representative is the marginalized kernel \cite{tsuda2002marginalized}, designed to deal with objects generated from latent variable models.
Given two visible variables, $x$ and $x'$ and two hidden variables, $h$ and $h'$, at first, a joint kernel $K_z(z,z')$ is defined over the two combined variables $z = (x,h)$ and $z' = (x',h')$. 
Then, a marginalized kernel for visible data is derived from the expectation with respect to hidden variables: $K(x,x') = \sum_h \sum_{h'} p(h|x)p(h'|x')K_z(z,z')$. 
The posterior distributions are in general unknown and are estimated by fitting a parametric model on the data.

In several cases, the assumption of a single parametric model underlying all the data may be too strong.
Additionally, finding the most suitable parametric model is a crucial and often difficult task, which must be repeated every time a new dataset is processed.
This issue is addressed by the autoregressive kernel \cite{cuturi2011autoregressive}, which evaluates the similarity of two time series on the corresponding likelihood profiles of a vector autoregressive model of a given order, across all possible parameter settings, controlled by a prior. 
The kernel is then evaluated as the dot product in the parameter space of such profiles, used as sequence representations. 
The reservoir based kernels~\cite{chen2013model}, map the time series into a high dimensional, dynamical feature space, where a linear readout is trained to discriminate each signal.
These kernels fit reservoir models sharing the same fixed reservoir topology to all time series. 
Since the reservoir provides a rich pool of dynamical features, it is considered to be ``generic'' and, contrarily to kernels based on a single parametric model, it is able to represent a wide variety of dynamics for different datasets. 

The methodology we propose is related to this last class of kernels. In order to create the TCK, we fuse the framework of representing time series via parametric models with similarity and kernel based methods. More specifically, the TCK leverages an ensemble of multiple models that, while they share the same parametric form, are trained on different subset of data, each time with different, randomly chosen initial conditions.

\section{Background}
\label{sec:Background}

In this section we provide a brief background on kernels, introduce the notation adopted in the remainder of the paper and provide the frameworks that our method builds on.
More specifically, we introduce the diagonal covariance GMM for MTS with missing data, the extended GMM framework with empirical priors and the related procedure to estimate the parameters of this model.

\subsection{Background on kernels}
Thorough overviews on kernels can be found in \cite{steinwart2008support, scholkopf2001learning, berg1984harmonic, shawe2004kernel}. Here we briefly review some basic definitions and properties, following \cite{steinwart2008support}.
\begin{mydef} \label{def 1}
  Let $\mathcal{X}$ be a non-empty set. 
  A function $k: \: \mathcal{X} \times \mathcal{X} \to \mathbb{R} $ is a \textit{kernel} if there exists a $\mathbb{R}$-Hilbert space $\mathcal{H}$ and a map $ \Phi: \: \mathcal{X} \to \mathcal{H}$ such that $\forall x, \: y \in \mathcal{X}$,
  $k(x,y) = \langle \Phi (x)  , \Phi(y) \rangle_{\mathcal{H}}.$
\end{mydef}
From this definition it can be shown that a kernel is symmetric and psd, meaning that $ \forall n \geq 1$, $\forall (a_1, \dots, a_n) \in \mathbb{R}^n$, $\forall (x_1,\dots,x_n) \in \mathcal{X}^n$,  $\sum_{i,j} a_i a_j K(x_i,x_j) \geq 0$. 
Of major importance in kernel methods are also the concepts of
reproducing kernels and reproducing kernel Hilbert spaces (RKHS), described by the following definition.
\begin{mydef} \label{def 2}
Let $\mathcal{X}$ be a non-empty set, $\mathcal{H}$ a Hilbert space and  $k: \: \mathcal{X} \times \mathcal{X} \to \mathbb{R}$ a function. $k$ is a \textit{reproducing kernel}, and $ \mathcal{H}$ a \textit{RKHS}, if $\forall x \in \mathcal{X}$, $\forall f \in \mathcal{H}$,  $k(\cdot,x) \in \mathcal{H}$ and $\langle f,  k(\cdot,x) \rangle_{\mathcal{H}} = f(x)$ (reproducing property).
\end{mydef}
These concepts are highly connected to kernels. In fact reproducing kernels are kernels, and every kernel is associated with a unique RKHS (Moore-Aronszajn theorem), and vice-versa.
Moreover, the \textit{representer theorem} states that every function in an RKHS that optimizes an empirical risk function can be expressed as a linear combination of kernels centered at the training points.  These properties have very useful implications, e.g. in an SVM, since an infinite dimensional empirical risk minimization problem can be simplified to a finite dimensional problem and the solution is included in the linear span of the kernel function evaluated at the training points.

\subsection{MTS with missing data}
We define a UTS, $x$, as a sequence of real numbers ordered in time,
$
x = \{x(t) \in \mathbb{R} \: | \: t = 1,2,\dots,T \}.
$
The independent time variable, $t$, is without loss of generality assumed to be discrete and the number of observations in the sequence, $T$, is the \textit{length} of the UTS. 

A  MTS $X$ is defined as a (finite) sequence of UTS,
$
X = \{ x_v \in \mathbb{R}^T \: | \: v = 1,2,\dots,V\},
$
where each attribute, $x_v$, is a UTS of length $T$. The number of UTS, $V$, is the \textit{dimension} of $X$. The length $T$ of the UTS $x_v$ is also the length of the MTS $X$. Hence,  a $V$--dimensional MTS, $X$, of length $T$ can be represented as a matrix in $\mathbb{R}^{V \times T}$. 

Given a dataset of $N$ MTS, we denote $X^{(n)}$ the $n$-th MTS. 
An incompletely observed MTS is described by the pair $(X^{(n)}, R^{(n)})$, where $R^{(n)}$ is a binary MTS 
with entry $r_v^{(n)}(t) = 0$ if the realization $x_v^{(n)}(t)$ is missing and $r_v^{(n)}(t) = 1$ if it is observed.

\subsection{Diagonal covariance GMM for MTS with missing data}
A GMM is a mixture of $G$ components, with each component belonging to a normal distribution. Hence, the components are described by the mixing coefficients $\theta_g$,  means $\mu_g$  and covariances  $\Sigma_{g}$. 
The mixing coefficients $\theta_g$ satisfy $0 \leq \theta_g \leq 1$ and $ \sum_{g=1}^G \theta_g = 1$. 

We formulate the GMM in terms of a latent random variable $ Z $, represented as a $G$-dimensional one-hot vector, whose marginal distribution is given by 
$
p(Z \: | \: \Theta ) = \prod\limits_{g=1}^G  \theta_g^{Z_g}.
$
The conditional distribution for the MTS $X$, given $Z$, is a multivariate normal distribution,
$
p(X \: | \: Z_g = 1, \: \Theta ) =  \mathcal{N} \left(X \: | \: \mu_{g}, \Sigma_{g}\right).
$
Hence, the GMM can be described by its probability density function (pdf), given by 
\begin{equation} \label{eq: p(x) gmm}
p(X) = \sum_Z p(Z) p(X \: | \:  Z, \: \Theta ) = \sum_{g=1}^G \theta_g \mathcal{N} \left(X \: | \: \mu_g, \Sigma_{g}\right).
\end{equation}
The GMM described by Eq.~\eqref{eq: p(x) gmm} holds for completely observed data and a general covariance. However, in the diagonal covariance GMM considered in this work, the following assumptions are made.  The MTS are characterized by time-dependent means, expressed by $\mu_g = \{ \mu_{gv} \in  \mathbb{R}^T \: | \: v = 1,...,V\}$, where  $\mu_{gv}$ is a UTS, whereas the covariances are constrained to be constant over time.  
Accordingly, the covariance matrix is $\Sigma_g = diag\{\sigma_{g1}^2,...,\sigma_{gV}^2\}$, being $\sigma_{gv}^2$ the variance of attribute $v$.
Moreover, the data is assumed to be \textit{missing at random} (MAR), i.e. the missing elements are only dependent on the observed values.

Under these assumptions, missing data can be analytically integrated away, such that imputation is not needed~\cite{rubin1976inference}, and the pdf for the incompletely observed MTS $(X, R)$ is given by
\begin{equation} \label{eq: p(x) gmm diag}
p(X \: | \: R, \: \Theta ) = \sum_{g=1}^G \theta_g \prod_{v=1}^V \prod_{t=1}^T  \mathcal{N} (x_v(t) \: | \: \mu_{gv}(t), \sigma_{gv})^{r_v(t) }
\end{equation}
The conditional probability of $Z$ given $X$, can be found using Bayes' theorem,
\begin{equation} \label{eq: p(z|x) posterior}
\pi_{g} \equiv P(Z_g = 1 \: | \: X, \: R, \: \Theta )  
= \frac{ \theta_g \prod_{v=1}^V \prod_{t=1}^T  \mathcal{N} \left(x_v(t) \: | \: \mu_{gv}(t), \sigma_{gv}\right)^{r_v(t) }}{\sum_{g=1}^G \theta_g \prod_{v=1}^V \prod_{t=1}^T  \mathcal{N} \left(x_v(t) \: | \: \mu_{gv}(t), \sigma_{gv}\right)^{r_v(t) }}.
\end{equation}
$\theta_g$ can be thought of as the prior probability of $X$ belonging to component $g$, and therefore Eq.~\eqref{eq: p(z|x) posterior} describes the corresponding posterior probability. 

To fit a GMM to a dataset, one needs to learn the parameters $ \Theta =  \{ \theta_g, \: \mu_{g}, \sigma_{g} \}_ {g=1}^G $. The standard way to do this is to perform maximum likelihood expectation maximization (EM)~\cite{bilmes1998gentle}. 
However, to be able to deal with large amounts of missing data, one can introduce informative priors for the parameters and estimate them using maximum a posteriori expectation maximization (MAP-EM)~\cite{Marlin:2012:UPD:2110363.2110408}.  
This ensures each cluster mean to be smooth over time and clusters containing few time series, to have parameters similar to the mean and covariance computed over the whole dataset.
We summarize this procedure in the next subsection (see Ref. \cite{Marlin:2012:UPD:2110363.2110408} for details).
 
\subsection{MAP-EM diagonal covariance GMM augmented with empirical prior}
To enforce smoothness,  a kernel-based Gaussian prior is defined for the mean,
$
P(\mu_{gv}) = \mathcal{N} \left(\mu_{gv} \: | \: m_{v}, \: S_{v}\right).
$
$m_{v}$ are the empirical means and the prior covariance matrices, $S_{v}$, are defined as
$
S_{v} = s_{v} \mathcal{K},
$
where $s_{v}$ are empirical standard deviations and $\mathcal{K}$ is a kernel matrix, whose elements are
$
\mathcal{K}_{tt'} = b_0 \exp (-a_0(t-t')^2), \quad t, \, t' = 1,\dots,T.
$
$a_0$, $b_0$ are user-defined hyperparameters.
 An inverse Gamma distribution prior is put on the standard deviation $\sigma_{gv}$,
$
P(\sigma_{gv}) \propto \sigma_{gv}^{-N_0} \exp \left(- \frac{N_0 s_v}{2 \sigma_{gv}^2} \right),
$
where $N_0$ is a user-defined hyperparameter. We denote $\Omega = \{ a_0, b_0, N_0\}$ the set of hyperparameters.
Estimates of parameters $\Theta$ are found using MAP-EM ~\cite{dempster1977maximum, mclachlan2007algorithm}, according to Algorithm \ref{alg:algorithm 2}.
%
\begin{algorithm}[!ht]
\small
\caption{MAP-EM diagonal covariance GMM}
\label{alg:algorithm 2}
\begin{algorithmic}[1]
\Require Dataset $\{(X^{(n)}, R^{(n)} )  \}_{n=1}^N$, hyperparameters $\Omega$ and number of mixtures $G$.
\State Initialize the parameters $\Theta$.
\State E-step. For each MTS $X^{(n)}$, evaluate the posterior probabilities using current parameter estimates,
$
\pi_{g}^{(n)} 
= P(Z_g = 1 \: | \: X^{(n)}, \: R^{(n)}, \: \Theta ).
$
\State M-step. Update parameters using the current posteriors
\begin{align*}
\theta_g &= N^{-1} \textstyle \sum_{n=1}^{N} \pi_{g}^{(n)} 
\\
\sigma_{gv}^2 &= \bigg( N_0 +  \sum_{n=1}^N \sum_{t=1}^T r^{(n)}_v(t) \; \pi_{g}^{(n)}\, \bigg)^{-1} \bigg( N_0 s^2_{v} + \sum_{n=1}^N \sum_{t=1}^T r^{(n)}_v(t) \; \pi^{(n)}_{g} \big(x^{(n)}_v(t) - \mu_{gv}(t)\big)^2 \bigg) 
\\
\mu_{gv} &= \left( S^{-1}_{v} + \sigma^{-2}_{gv} \textstyle \sum\limits_{n=1}^N  \pi^{(n)}_{g} \text{diag}(r^{(n)}_{v} ) \right)^{-1} 
  \left( S^{-1}_{v} m_{v} +  \sigma^{-2}_{gv} \textstyle \sum\limits_{n=1}^N  \pi^{(n)}_{g} \text{diag}(r^{(n)}_{v} ) \: x^{(n)}_v \right)   
\end{align*} 
\State Repeat step 2-3 until convergence.
\Ensure Posteriors $ \Pi^{(n)} \equiv \left( \pi_1^{(n)},\dots,\pi_G^{(n)} \right)^T $ and mixture parameters $ \Theta $.
\end{algorithmic}
\end{algorithm}
%
\section{Time series cluster kernel (TCK)}
\label{sec:TCK}
Methods based on GMM, in conjunction with EM, have been successfully applied in different contexts, such as density estimation and clustering~\cite{statlearning}.   
As a major drawback, these methods often require to solve a non-convex optimization problem, whose outcome depends on the initial conditions~\cite{mclachlan2007algorithm, wu1983convergence}.
The model described in the previous section depends on initialization of parameters $\Theta$ and the chosen number of clusters $G$~\cite{Marlin:2012:UPD:2110363.2110408}. Moreover, three different hyper-parameters, $a_0, b_0, N_0$, have to be set. In particular,  modeling the covariance in time is difficult; choosing a too small hyperparameter $a_0$ leads to a degenerate covariance matrix that cannot be inverted. 
On the other hand, a too large value would basically remove the covariance such that the prior knowledge is not incorporated. 
Furthermore, a single GMM provides a limited descriptive flexibility, due to its parametric nature.

Ensemble learning has been adopted both in classification, where classifiers are combined through e.g. bagging or boosting~\cite{breiman1996bagging, freund1996experiments}, and clustering~\cite{Fred02evidenceaccumulation,Monti2003,Strehl:2003}.  
Typically, in ensemble clustering one integrates the outcomes of the same algorithm as it processes different data subsets, being configured with different parameters or initial conditions, in order to capture local and global structures in the underlying data~\cite{Monti2003,vega2011survey} and to provide a more stable and robust final clustering result.
Hence, the idea is to combine the results of many weaker models to deliver an estimator with statistical, computational and representational advantages~\cite{Dietterich2000}, which are lower variance, lower sensitivity to local optima and a broader span of representable functions, respectively.

We propose an ensemble approach that combines multiple GMM, whose diversity is ensured by training the models on subsamples of data, attributes and time segments, using different numbers of mixture components and random initialization of $\Theta$ and hyperparameters. 
Thus, we generate a model robust to parameters and noise, also capable of capturing different levels of granularity in the data. 
To ensure robustness to missing data, we use the diagonal covariance GMM augmented with the informative priors described in the previous section as base models in the ensemble.

Potentially, we could have followed the idea of \cite{glodek2013ensemble}  to create a density function from an ensemble of GMM. 
Even though several methods rely on density estimation~\cite{statlearning}, we aim on deriving a \textit{similarity measure}, which provides a general-purpose data representation, fundamental in many applications in time-series analysis, such as classification, clustering, outlier detection and dimensionality reduction~\cite{han2011data}.

Moreover, we ensure the similarity measure to be psd, i.e. a \textit{kernel}. 
Specifically, the linear span of posterior distributions $ \pi_{g} $, formed as  $G$-vectors, with ordinary inner product, constitutes a Hilbert space. 
We explicitly let the \textit{feature map} $\Phi$ be these posteriors. 
Hence, the TCK is an inner product between two distributions and therefore forms a linear kernel in the space of posterior distributions. Given an ensemble of GMM, we create the TCK using the fact that the sum of kernels is also a kernel.

\subsection{Method details}
To build the TCK kernel matrix, we first fit different diagonal covariance GMM to the MTS dataset.
To ensure diversity, each GMM model uses a number of components from the interval $[2,C]$. For each number of components, we apply $Q$ different  random initial conditions and hyperparameters. We let $\mathcal{Q} = \{ q = (q_1,q_2) \: | \: q_1=1,\dots Q, \: q_2 = 2,\dots, C \} $ be the index set keeping track of initial conditions and hyperparameters ($q_1$), and the number of components ($q_2$).
Moreover, each model is trained on a random subset of MTS, accounting only a random subset of variables $\mathcal{V}$, with cardinality $ |\mathcal{V}| \leq V$, over a randomly chosen time segment $\mathcal{T}, |\mathcal{T}| \leq T$. 
The inner products of the posterior distributions from each mixture component are then added up to build the TCK kernel matrix, according to the ensemble strategy~\cite{ensemble}.
Algorithm~\ref{alg:algorithm} describes the details of the method.

\begin{algorithm}[h!]
\small
\caption{TCK kernel. Training phase.}
\label{alg:algorithm}
\begin{algorithmic}[1]
\Require Training data $ \{ (X^{(n)}, R^{(n)} )  \}_{n=1}^N$ , $Q$ initializations, $C$ maximal number of mixture components.
\State Initialize kernel matrix $K = 0_{N \times N}  $.
\For{$q \in \mathcal{Q}$}
\State Compute posteriors $ \Pi^{(n)}(q) \equiv \left( \pi_1^{(n)},\dots,\pi_{q_2}^{(n)} \right)^T $, $n = 1,\dots, N$, by applying Algorithm~\ref{alg:algorithm 2} with $q_2$ clusters and by randomly selecting,
\begin{itemize}
\item[i.] hyperparameters $\Omega(q) $,
\item[ii.] a time segment $ \mathcal{T}(q)  $ of length $T_{min} \leq  |\mathcal{T}(q)| \: \leq \: T_{max}$,
\item[iii.] a subset of attributes, $\mathcal{V}(q) \subset (1,\dots,V) $, with cardinality $V_{min} \leq |\mathcal{V}(q)| \leq V_{max}$,
\item[iv.] a subset of MTS, $\eta(q) \subset (1,\dots,N) $, with cardinality $N_{min} \leq |\eta(q)| \leq N$,
\item[v.] initialization of the mixture parameters $ \Theta(q) $.
\end{itemize}
\State Update kernel matrix, $K_{nm} = K_{nm} + \Pi^{(n)}(q)^T \Pi^{(m)}(q) $,  $\quad n, \, m = 1,\dots , N$.
\EndFor
\Ensure $K$ TCK kernel matrix, time segments $\mathcal{T}(q)  $, subsets of attributes $\mathcal{V}(q)$, subsets of MTS $\eta(q)$, GMM parameters $ \Theta(q)$  and posteriors $\Pi^{(n)}(q) $.
\end{algorithmic}
\end{algorithm}

In order to be able to compute similarities with MTS not available at the training phase, one needs to store the time segments $\mathcal{T}(q)$, subsets of attributes $\mathcal{V}(q)$, GMM parameters $ \Theta(q)$  and posteriors $\Pi^{(n)}(q)$.
Then, the TCK for such out-of-sample MTS is evaluated according to Algorithm~\ref{alg:algorithm out of sample}.

\begin{algorithm}[h!]
\setstretch{1}
\small
\caption{TCK kernel. Test phase.}
\label{alg:algorithm out of sample}
\begin{algorithmic}[1]
  \Require Test set $\big \{ (X^{*(m)}, R^{*(m)} )\big \}_{m=1}^M$, time segments $\mathcal{T}(q)$, subsets of attributes $\mathcal{V}(q)$,  subsets of MTS $\eta(q)$, GMM parameters $  \Theta(q) $  and posteriors $\Pi^{(n)}(q) $. 
  \State Initialize kernel matrix $K^* = 0_{N \times M} $. 
  \For{$q \in \mathcal{Q}$}
  \State Compute posteriors $\Pi^{*(m)}(q) $, $m=1,\dots,M$ by applying Eq.~\eqref{eq: p(z|x) posterior}  with mixture parameters $ \Theta(q)$.
  \State Update kernel matrix, $K^*_{nm} = K^*_{nm} + \Pi^{(n)}(q)^T \Pi^{*(m)}(q) $, $\quad n = 1,\dots , N$, $m=1,\dots,M$.
  \EndFor
  \Ensure $K^*$ TCK test kernel matrix
\end{algorithmic}
\end{algorithm}

\subsection{Parameters and robustness} 
\label{sec: def param} 
The maximal number of mixture components in the GMM, $C$, should be set high enough to capture the local structure in the data. On the other hand, it should be set reasonably lower than the number of MTS in the dataset in order to be able to estimate the parameters of the GMM. 
Intuitively, a high number of realizations $Q$ improves the robustness of the ensemble of clusterings. 
However, more realizations comes at the expense of an increased computational cost. 
In the end of next section we show experimentally that it is not critical to correctly tune these two hyperparameters as they just have to be set high enough.

Through empirical evaluations we have seen that none the other hyperparameters are critical. We set default hyperparameters as follows.
The hyperparameters are sampled according to a uniform distribution from pre-defined intervals. Specifically, we let
$ a_0 \in (0.001,1)$, $b_0  \in (0.005,0.2)$ and $N_0  \in (0.001,0.2)$.
The subsets of attributes are selected randomly by sampling according to a uniform distribution from $\{2,\dots,V_{max}\}$. The lower bound is set to two, since we want to allow the algorithm to learn possible inter-dependencies between at least two attributes. The time segments are sampled from $ \{1, \dots, T\}$ and the length of the segments are allowed to vary between $T_{min}$ and $T_{max}$. In order to be able to capture some trends in the data we set $T_{min}= 6$. We let the minimal size of the subset of MTS be 80 percent of the dataset.

We do acknowledge that for long MTS the proposed method becomes computationally demanding, as the complexity scales as $\mathcal{O}(T^3)$. 
Moreover, there is a potential issue in Eq.~\eqref{eq: p(z|x) posterior} since multiplying together very small numbers both in the nominator and denominator could yield to numerically unstable expressions close to $0/0$. While there is no theoretical problem, since the normal distribution is never exactly zero, the posterior for some outliers could have a value close to the numerical precision.
In fact, since the posterior assignments are numbers lower than $1$, the value of their product can be small if $V$ and $T$ are large.
We address this issue by putting upper thresholds on the length of the time segments, $T_{max}$, and number of attributes, $V_{max}$, which is justified by the fact that the TCK is learned using an ensemble strategy.
Moreover, to avoid problems for outliers we put a lower bound on the value for the conditional distribution for $x_v(t)$ at $\mathcal{N} \left(3 \: | \: 0, 1\right)$. In fact, it is very unlikely that a data point generated from a normal distribution is more than three standard deviations away from the mean. 

\subsection{Algorithmic complexity}
\paragraph{Training complexity} 
The computational complexity of the EM procedure is dominated by the update of the mean, whose cost is $ \mathcal{O}(2T^3 + N V T^2 )$. 
Hence, for $G$ components and $I$ iterations, the total cost is $\mathcal{O}\left(I G (2T^3 + N V T^2 ) \right)$.
The computation of the TCK kernel involves both the MAP-EM estimation and the kernel matrix generation for each $q \in \mathcal{Q}$, whose cost is upper-bounded by $\mathcal{O}\left( N^2 C \right)$. 
The cost of a single evaluation $q$ is therefore bounded by $\mathcal{O}\left( N^2 C+ I C (2T_{max}^3 + N V_{max} T_{max}^2 ) \right)$. 
We underline that the effective computational time can be reduced substantially through parallelization, since each instance $q \in \mathcal{Q}$ can be evaluated independently.
As we can see, the cost has a quadratic dependence on $N$, which becomes the dominating term in large datasets.  
We note that in spectral methods the eigen-decomposition costs $\mathcal{O}(N^3)$ with a consequent complexity higher than TCK for large $N$. 

\paragraph{Testing complexity} 
For a test MTS one has to evaluate $|\mathcal{Q}|$ posteriors, with a complexity bounded by $\mathcal{O}(C V_{max} T_{max})$. 
The complexity of computing the similarity with the $N$ training MTS is bounded by $\mathcal{O}\left( N C \right)$. 
Hence, for each $q \in \mathcal{Q}$, the testing complexity is $\mathcal{O}(N C  + C V_{max} T_{max})$. 
Note that also the test phase is embarrassingly parallelizable.

\subsection{Properties}
In this section we demonstrate that TCK is a proper kernel and we discuss some of its properties. 
We let $\mathcal{X} = \mathbb{R}^{V \times T}$ be the space of $V$-variate MTS of length $T$ and $K: \: \mathcal{X} \times \mathcal{X} \to \mathbb{R}$ be the TCK. 

\begin{theorem}
  $K$ is a kernel.
\end{theorem}

\begin{proof}
According to the definition of TCK, we have
$
  K(X^{(n)},X^{(m)})  = \sum_{q \in \mathcal{Q}} k_q(X^{(n)},X^{(m)}),
$
where $k_q(X^{(n)},X^{(m)}) = \Pi^{(n)}(q)^T \Pi^{(m)}(q)$. 
Since the sum of kernels is a kernel, it is sufficient to demonstrate that $k_q$ is a kernel. 
We define 
$
\mathcal{H}_q =  \{ f = \sum_{n=1}^N \alpha_n \Pi^{(n)}(q)  \: \big| \: N \in \mathbb{N}, \: X^{(1)},\dots,X^{(N)} \in \mathcal{X}, \: \alpha_1,\dots,\alpha_N \in \mathbb{R} \}.
$
Since $\mathcal{H}_q$ is the linear span of posterior probability distributions, it is closed under addition and scalar multiplication and therefore a vector space.
Furthermore, we define an inner product in $\mathcal{H}_q $ as the ordinary dot-product in $\mathbb{R}^{q_2}$, $\langle f , f' \rangle_{\mathcal{H}_q} = f^T f'$. 

\begin{lemma} 
	$\mathcal{H}_q$ with $\langle \cdot  , \cdot \rangle_{\mathcal{H}_q}$ is a Hilbert space.
\end{lemma}

\begin{proof}
	$\mathcal{H}_q$ is equipped with the ordinary dot product, has finite dimension $q_2$ and therefore is isometric to $\mathbb{R}^{q_2}$ .
\end{proof}

\begin{lemma}
	$k_q$ is a kernel.
\end{lemma}

\begin{proof}
	Let $\Phi_q: \:  \mathcal{X} \to \mathcal{H}_q$ be the mapping given by $X \rightarrow \Pi(q)$. It follows that
$
\langle \Phi_q (X^{(n)})  , \Phi_q(X^{(m)}) \rangle_{\mathcal{H}_q} = \langle \Pi(q)^{(n)} , \Pi(q)^{(m)} \rangle_{\mathcal{H}_q} = (\Pi(q)^{(n)})^T \Pi(q)^{(m)} = k_q(X^{(n)},X^{(m)}).
$
\end{proof}

Now, let $ \mathcal{H}$ be the Hilbert space defined via direct sum, $ \mathcal{H} = \bigoplus\limits_{q \in \mathcal{Q}} \mathcal{H}_q $.
$\mathcal{H}$ consists of the set of all ordered tuples $\mathbf{\Pi}^{(n)} = (\Pi^{(n)}(1),\Pi^{(n)}(2), \dots, \Pi^{(n)}(|\mathcal{Q}|))$. 
An induced inner product on  $\mathcal{H}$ is 
$
  \langle \mathbf{\Pi}^{(n)}, \mathbf{\Pi}^{(m)} \rangle_{\mathcal{H}} = \sum_{q \in \mathcal{Q}} \langle \Pi^{(n)}(q)  , \Pi^{(m)}(q) \rangle_{\mathcal{H}_q}.
$
If we let $\Phi: \:  \mathcal{X} \to \mathcal{H} $ be the mapping given by $X^{(n)} \rightarrow \mathbf{\Pi}^{(n)}$, it follows that
$
\langle \Phi (X^{(n)})  , \Phi(X^{(m)}) \rangle_{\mathcal{H}} = \langle \mathbf{\Pi}^{(n)}, \mathbf{\Pi}^{(m)} \rangle_{\mathcal{H}} = 
  \sum_{q \in \mathcal{Q}} k_q(X^{(n)},X^{(m)}) = K(X^{(n)},X^{(m)}). 
$
\end{proof}
This result and its proof unveil important properties of TCK. 
(i) $K$ is symmetric and psd; 
(ii) the feature map $\Phi$ is provided explicitly;
(iii) $K$ is a linear kernel in the Hilbert space of posterior probability distributions $\mathcal{H}$;
(iv) the induced distance $d$, given by
\begin{align*}
d^2(X^{(n)}, X^{(m)}) &= \langle \Phi (X^{(n)}) - \Phi(X^{(m)})  , \Phi(X^{(m)}) - \Phi(X^{(m)}) \rangle_{\mathcal{H}} \\
&= K(X^{(n)},X^{(n)}) - 2 K(X^{(n)},X^{(m)}) + K(X^{(m)},X^{(m)})
\end{align*}
is a pseudo-metric as it satisfies the triangle inequality, takes non-negative values, but, in theory, it can vanish for $ X^{(n)} \neq X^{(m)}$.

\section{Experiments and results}
\label{sec:experiments}

The proposed kernel is very general and can be used as input in many learning algorithms. It is beyond the scope of this paper to illustrate all properties and possible applications for TCK. Therefore we restricted ourselves to classification, with and without missing data, dimensionality reduction and visualization.
We applied the proposed method to one synthetic and several benchmark datasets.
The TCK was compared to three other similarity measures, DTW, LPS and the fast global alignment kernel (GAK)~\cite{cuturi2011fast}.
DTW was extended to the multivariate case using both the \textit{independent} (DTW i) and \textit{dependent} (DTW d) version~\cite{shokoohi2015generalizing}. 
To evaluate the robustness of the similarity measures, they were trained unsupervisedly also in classification experiments, without tuning hyperparameters by cross-validation.
In any case, cross-validation is not trivial in multivariate DTW, as the best window size based on individual attributes is not well defined~\cite{baydogan2016time}. 

For the classification task, to not introduce any additional, unnecessary parameters, we chose to use a nearest-neighbor (1NN) classifier. This is a standard choice in time series classification literature~\cite{Kate2016}. 
Even though the proposed method provides a kernel, by doing so, it is easier to compare the different properties of the  similarity measures directly to each other. Performance was measured in terms of \textit{classification accuracy} on a test set.  

To perform dimensionality reduction we applied kPCA using the two largest eigenvalues of the kernel matrices. 
The different kernels were visually assessed by plotting the resulting mappings with the class information color-coded.
  
The TCK was implemented in R and Matlab, and the code is made publicly available at~\cite{MikalsenTCK}. In the experiments we used the same parameters on all datasets. We let $C = 40$ and $Q = 30$. For the rest of the parameters we used the default values discussed in Section~\ref{sec: def param}. The only exception is for datasets with less than 100 MTS, in that case we let the maximal number of mixtures be $C = 10$. The hyperparameter dependency is discussed more thoroughly in the end of this section.

For the LPS we used the Matlab implementation provided by Baydogan~\cite{Baydogan}.
We set the number of trees to 200 and number of segments to 5. Since many of the time series we considered were short, we set the minimal segment length to 15 percent of the length of MTS in the dataset. The remaining hyperparameters were set to default. For the DTW we used the \textit{R} package \textit{dtw}~\cite{Giorgino}. The GAK was run using the Matlab Mex implementation provided by Cuturi~\cite{Cuturi}. In accordance with~\cite{Cuturi} we set the bandwidth $\sigma$ to two times the median distance of the MTS in the training set, scaled by the square root of the median length of the MTS. The triangular parameter was set to 0.2 times the median length.

In contrast to the TCK and LPS, the
DTW and GAK do not naturally deal with missing data and therefore we imputed the overall mean for each attribute and time interval.

\subsection{Synthetic example: Vector autoregressive model}

We first applied TCK in a controlled experiment, where we generated a synthetic MTS dataset with two classes from a first-order vector autoregressive model, VAR(1)~\cite{shumway}, given by 
\begin{align}
\begin{pmatrix}
x_1(t) \\
x_2(t)
\end{pmatrix}
=
\begin{pmatrix}
\alpha_1 \\
\alpha_2
\end{pmatrix}
+
\begin{pmatrix}
\rho_x & 0\\
0 & \rho_y
\end{pmatrix}
\begin{pmatrix}
x_1(t-1) \\
x_2(t-1)
\end{pmatrix}
+
\begin{pmatrix}
\xi_1(t) \\
\xi_2(t)
\end{pmatrix}
\end{align}
To make $x_1(t)$ and $x_2(t)$ correlated with $\mathrm{corr}(x_1(t),x_2(t)) = \rho$, we chose the noise term s.t., 
$
\mathrm{corr}\left(\xi_1(t),\xi_2(t)\right) = \rho \: (1- \rho_x \rho_y) \: [(1-\rho_x^2)(1-\rho_y^2)]^{-1}.
$
For the first class, we generated 100 two-variate MTS of length 50 for the training and 100 for the test, from the VAR(1)-model with parameters $\rho = \rho_x = \rho_y = 0.8$ and $\mathbb{E}[(x_1(t), x_2(t))^T] = (0.5, -0.5)^T$.
Analogously, the MTS of the second class were generated using parameters $\rho = -0.8$, $\rho_x = \rho_y = 0.6$ and $\mathbb{E}[(x_1(t), x_2(t))^T] = (0, 0)^T$.
On these synthetic data, in addition to dimensionality reduction and classification with and without missing data, we also performed spectral clustering on the TCK matrix in order to be able to compare TCK directly to a single diagonal covariance GMM optimized using MAP-EM.

\paragraph{Clustering}
Clustering performance was measured in terms of \textit{adjusted rand index} (ARI)~\cite{Hubert1985} and \textit{clustering accuracy} (CA).
CA is the maximum bipartite matching ($map$) between cluster labels ($l_i$) and ground-truth labels ($y_i$), defined as
$
\mathrm{CA} = N^{-1} \sum_{i=1}^N \delta (y_i, \mathrm{map}(l_i)),
$
where $\delta(\cdot,\cdot)$ is the Kronecker delta and $\mathrm{map}(\cdot)$ is computed with the Hungarian algorithm~\cite{Kuhn55thehungarian}.

\begin{SCtable*}[0.5][!t]
\centering
\caption{Clustering performance, measured in terms of CA and ARI, on simulated VAR(1) datasets for TCK and GMM.}\label{tab: var(1) class and clust}
\begin{tabular}{lllll}
\cmidrule[1.5pt]{1-5}
 \multicolumn{1}{c}{} &  TCK  & GMM & TCK$_{UTS}$ & TCK$_{\rho = 0}$ \\ 
\cmidrule[.5pt]{1-5}
CA & 0.990 &0.910 &0.775 & 0.800 \\
ARI  & 0.961 & 0.671  & 0.299 & 0.357 \\ 
\cmidrule[1.5pt]{1-5}
\end{tabular}
\end{SCtable*}
 
The single GMM was run with $a_0 = 0.1$, $b_0= 0.1$ and $N_0 = 0.01$.  
Tab.~\ref{tab: var(1) class and clust} show that spectral clustering on the TCK achieves a considerable improvement compared to GMM clustering and verify the efficacy of the ensemble and the kernel approach with respect to a single GMM.  
Additionally, we evaluated TCK by concatenating the MTS as a long vector and thereby treating the MTS as an UTS (TCK$_{UTS}$) and on a different VAR(1) dataset with the attributes uncorrelated (TCK$_{\rho = 0}$). The superior performance of TCK with respect to these two approaches illustrates that, in addition to accounting for similarities within the same attribute, TCK also leverages interaction effects between different attributes in the MTS to improve clustering results.

\paragraph{Dimensionality reduction and visualization}
%
\begin{SCfigure*}[0.5][!t]
\centering
  \includegraphics[width=0.78\textwidth, keepaspectratio=true]{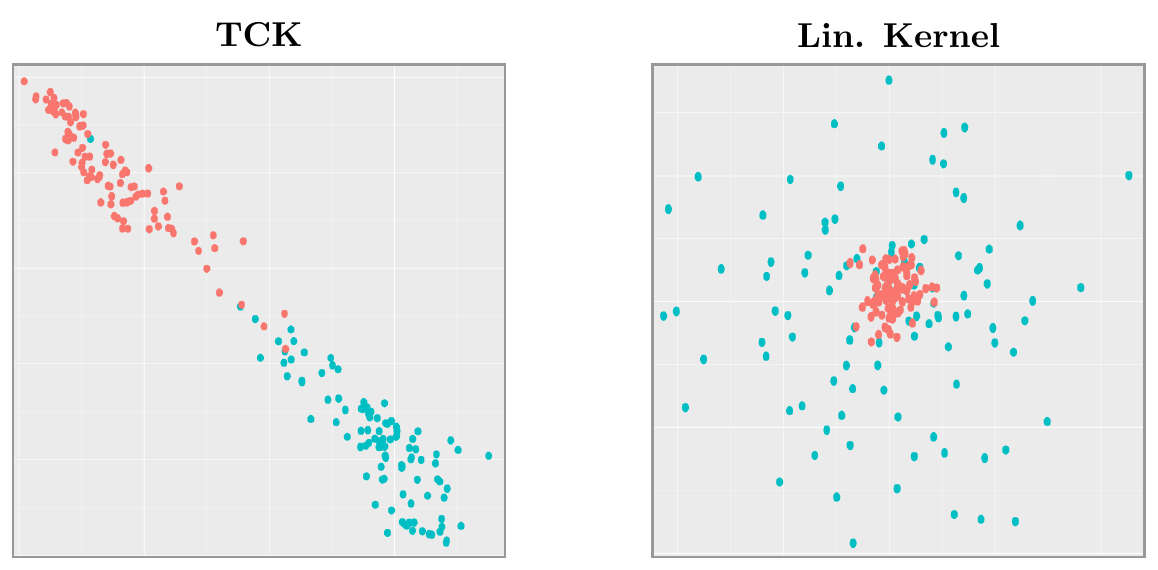}
\caption{Projection of the VAR(1) dataset to two dimensions using kPCA with the TCK and a linear kernel. The different colors indicate the true labels of the MTS.}
\label{fig:var(1)_2d_embedding}
\end{SCfigure*}
%
To evaluate the effectiveness of TCK as a kernel, we compared kPCA with TCK and kPCA with a linear kernel (ordinary PCA). 
Fig.~\ref{fig:var(1)_2d_embedding} shows that TCK maps the MTS on a line, where the two classes are well separated.  
On the other hand, PCA projects one class into a compact blob in the middle, whereas the other class is spread out. 
Learned representations like these can be exploited by learning algorithms such as an SVM.
In this case, a linear classifier will perform well on the TCK representation, whereas for the other representation a non-linear method is required. 

\paragraph{Classification with missing data}

To investigate the TCK capability of dealing with missing data in a classification task, we removed values from the synthetic dataset according to three missingness patterns: \textit{missing completely at random} (MCAR), \textit{missing at random} (MAR) and \textit{missing not at random} (MNAR)~\cite{rubin1976inference}. 
To simulate MCAR, we uniformly sampled the elements to be removed. 
Specifically, we discarded a ratio $p_{MCAR}$ of the values in the dataset, varying from $0$ to $0.5$. 
To simulate MAR, we let $x_i(t)$ have a probability $p_{MAR}$ of being missing, given that $x_j(t) > 0.5$, $i \neq j$. 
Similarly, for MNAR we let $x_i(t)$ have a probability $p_{MNAR}$ of being missing, given that $x_i(t) > 0.5$. We varied the probabilities from  0 to  0.5 to obtain different fractions of missing data. 

\begin{figure}[!t]
\centering
  \subfigure
  {
  \includegraphics[width=0.313\columnwidth, trim={0mm 1mm 1mm 0mm},clip]{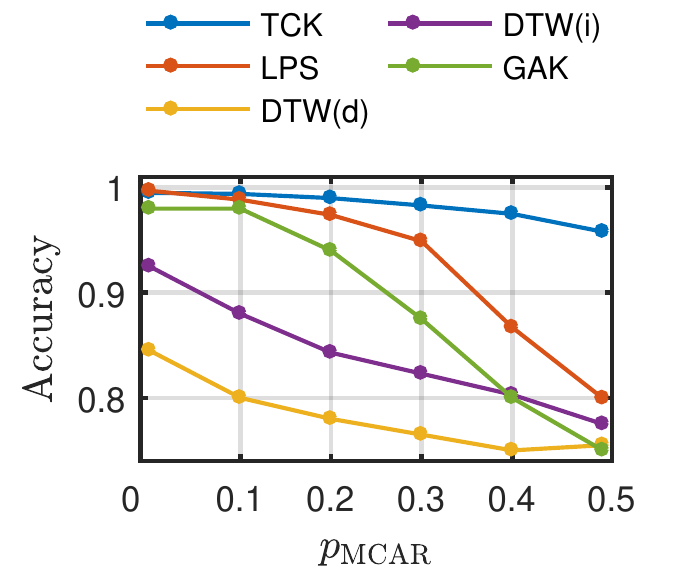}
  \label{fig:missing_VAR1_MCAR}}\hspace{0em}%
  \subfigure
  {
  \includegraphics[width=0.313\columnwidth, trim={0mm 1mm 1mm 0mm},clip]{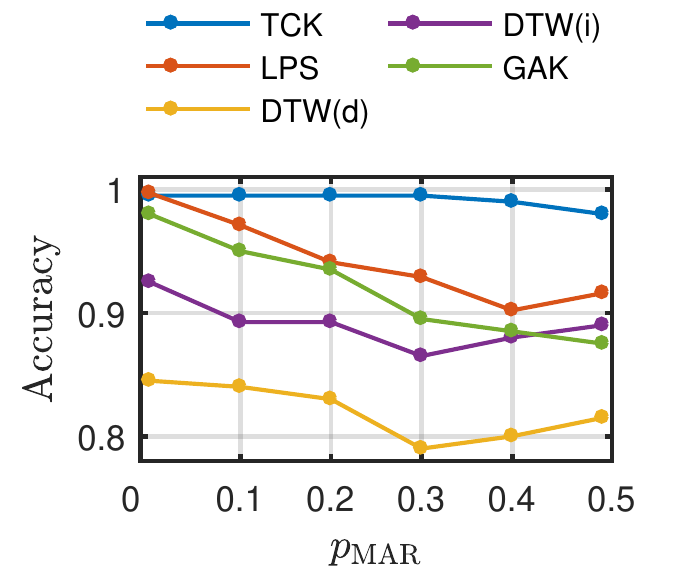}
  \label{fig:missing_VAR1_MAR}}\hspace{0em}%
  \subfigure
  {
  \includegraphics[width=0.313\columnwidth, trim={0mm 1mm 1mm 0mm},clip]{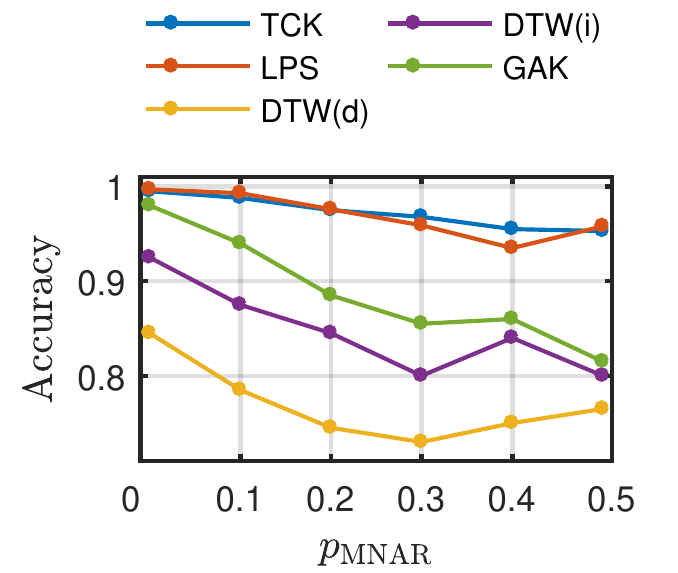}
  \label{fig:missing_VAR1_MNAR}}\hspace{0em}%
\caption{Classification accuracy on simulated VAR(1) dataset of the 1NN-classifier configured with a (dis)similarity matrix obtained using LPS, DTW (d), DTW (i), GAK and TCK. 
We report results for three different types of missingness, with an increasing percentage of missing values.}
\label{fig:missing_VAR1}
\end{figure}

For each missingness pattern, we evaluated the performance of a 1NN classifier configured with TCK, LPS, DTW (d), DTW (i) and GAK.
Classification accuracies are reported in Fig.~\ref{fig:missing_VAR1}.
First of all, we see that in absence of missing data, the performance of TCK and LPS are approximately equal, whereas the two versions of DTW and GAK yield a lower accuracy.
Then, we notice that the accuracy for the TCK is quite stable as the amount of missing data increases, for all types of missingness patterns. 
For example, in the case of MCAR, when the amount of missing data increases from 0 to 50\%, accuracy decreases to from 0.995 to 0.958. 
Likewise, when $p_{MNAR}$ increases from 0 to 0.5, accuracy decreases from 0.995 to 0.953. 
This indicates that our method, in some cases, also works well for data that are MNAR. 
On the other hand, we notice that for MCAR and MAR data, the accuracy obtained with LPS decreases much faster than for TCK. GAK seems to be sensitive to all three types of missing data.
Performance also diminishes quite fast in the DTW variants, but we also observe a peculiar behavior as the accuracy starts to increase again when the missing ratio increases.
This can be interpreted as a side effect of the imputation procedure implemented in DTW.
In fact, the latter replaces some noisy data with a mean value, hence providing a regularization bias that benefits the classification procedure.

\subsection{Benchmark time series datasets}
\begin{table*}[!t]
\small
\centering
\caption{Description of benchmark time series datasets. Column 2 to 5 show the number of attributes, samples in training and test set, and classes, respectively. $T_{min}$ is the length of the shortest MTS in the dataset and $T_{max}$ the longest MTS. $T$ is the length of the MTS after the transformation.}\label{tab: benchmark description}
\begin{tabular}{@{}ll@{ }@{ }l@{ }@{ }@{ }l@{ }@{ }@{ }ll@{ }l@{ }lc@{}}
\cmidrule[1.5pt]{1-9}
Datasets & Attributes &  Train  & Test  &  Classes  & $T_{min}$ & $T_{max}$ & $T$ & Source \\
\cmidrule[.5pt]{1-9}
ItalyPower & 1 & 67 & 1029 &2 & 24 & 24 & 24 & UCR \\
Gun Point & 1 & 50 & 150 & 2 & 150 & 150 & 150 & UCR \\
Synthetic control & 1 & 300 & 300 & 6 & 60 & 60 & 60 & UCR \\
\cmidrule[.5pt]{1-9}
PenDigits &  2 &  300 & 10692 & 10  & 8 & 8 & 8 & UCI \\
Libras & 2 & 180 & 180 & 15 & 45 & 45 & 23 &   UCI \\
ECG & 2 & 100 & 100 & 2 & 39 & 152 & 22 & Olszewski \\
uWave & 3 & 200 & 4278 & 8 & 315 & 315 & 25 & UCR \\
Char.Traj. & 3 & 300 & 2558 & 20 & 109 & 205 & 23 & UCI \\
Robot failure LP1 & 6 & 38 & 50 & 4 & 15 & 15 & 15 & UCI \\ 
Robot failure LP2 &  6 & 17 & 30 & 5 & 15 & 15 & 15 & UCI  \\ 
Robot failure LP3 & 6 & 17 & 30 & 4 & 15 & 15 & 15  & UCI \\ 
Robot failure LP4 & 6 & 42 & 75 & 3 & 15 & 15 & 15  & UCI \\ 
Robot failure LP5 & 6 & 64 & 100 & 5 & 15 & 15 & 15  & UCI\\ 
Wafer & 6 & 298 & 896 & 2 & 104 & 198 & 25 & Olszewski \\
Japanese vowels & 12 & 270 & 370 & 9 & 7 & 29 & 15 & UCI \\
ArabicDigits & 13 & 6600 & 2200 & 10 & 4 & 93 & 24 & UCI \\
CMU & 62 & 29 & 29 & 2 & 127 & 580 & 25 & CMU \\
PEMS & 963 & 267 & 173 & 7 & 144 & 144 & 25 & UCI \\
\cmidrule[1.5pt]{1-9}
\end{tabular}
\end{table*}

We applied the proposed method to multivariate benchmark datasets from the UCR and UCI databases~\cite{UCRArchive, Lichman:2013} and other published work~\cite{CMU, 
Olszewski}, described in Tab.~\ref{tab: benchmark description}. In order to also illustrate TCK's capability of dealing with UTS, we randomly picked three univariate datasets from the UCR database; \textit{ItalyPower}, \textit{Gun Point} and \textit{Synthetic control}.
Some of the multivariate datasets contain time series of different length. 
However, the proposed method is designed for MTS of the same length.
Therefore we followed the approach of Wang et al.~\cite{Wang2016237} and transformed all the MTS in the same dataset to the same length, $T$, determined by 
$
T = \left \lceil \frac{T_{max}}{\left \lceil \frac{T_{max}}{25} \right \rceil} \right \rceil,
$
where $T_{max}$ is the length of the longest MTS in the dataset and $ \lceil \: \rceil$ is the ceiling operator. 
We also standardized to zero mean and unit standard deviation. 
Since decision trees are scale invariant, we did not apply this transformation for LPS (in accordance with~\cite{baydogan2016time}).

\paragraph{Classification without missing data}
Initially we considered the case of no missing data and applied a 1NN-classifier in combination with the five different (dis)similarity measures. Tab.~\ref{tab: results} shows the mean classification accuracies, evaluated over 10 runs, obtained on the 
benchmark time series datasets. Firstly, we notice that the dependent version of DTW, in general, gives worse results than the independent version. 
Secondly, TCK gives the best accuracy for 8 out of 18 datasets. LPS and GAK are better than the competitors for 8 and 3 datasets, respectively. The two versions of DTW achieve the highest accuracy for Gun Point. On CMU all methods reach a perfect score. 
We also see that TCK works well for univariate data and gives comparable accuracies to the other methods.

\begin{SCtable*}[0.3][!t]
\small
\centering
\caption{Classification accuracy on different UTS and MTS benchmark datasets obtained using TCK, LPS, DTW (i), DTW (d) and GAK in combination with a 1NN-classifier. The best results are highlighted in bold. }
\label{tab: results}
\begin{tabular}{llllll}
\cmidrule[1.5pt]{1-6}
Datasets &  TCK  & LPS  & DTW (i) & DTW (d) & GAK \\
\cmidrule[.5pt]{1-6}
ItalyPower &   0.922&   0.933  & 0.918 & 0.918 & \textbf{0.950} \\ 
Gun Point &   0.923 & 0.790 &  \textbf{1.000} & \textbf{1.000} & 0.900 \\ 
Synthetic control &   \textbf{0.987} & 0.975     & 0.937 & 0.937 & 0.870 \\ 
\cmidrule[.5pt]{1-6}
Pen digits &   0.904 & 0.928 &  0.883 & 0.900 & \textbf{0.945} \\ 
Libras &  0.799 &  \textbf{0.894} & 0.878   & 0.856 & 0.811\\ 
ECG &     \textbf{0.852}  & 0.815 &  0.810 & 0.790 & 0.840 \\ 
uWave   &  0.908 & \textbf{0.945}  & 0.909 & 0.844 & 0.905\\ 
Char. Traj.  & 0.953 & \textbf{0.961} &  0.903 & 0.905 & 0.935 \\ 
Robot failure LP1  & \textbf{0.890} & 0.836 & 0.720  & 0.640 & 0.720\\ 
Robot failure LP2   & 0.533 & \textbf{0.707}  &0.633 & 0.533 & 0.667\\  
Robot failure LP3 &  \textbf{0.703}  &  0.687  &0.667  &  0.633 & 0.633 \\ 
Robot failure LP4   & 0.848 & \textbf{0.914} &  0.880  & 0.840 & 0.813\\ 
Robot failure LP5 & 0.596 & \textbf{0.688} & 0.480 & 0.430 & 0.600\\ 
Wafer &    \textbf{0.982} & 0.981  & 0.963 & 0.961 & 0.967 \\ 
Japanese vowels   & \textbf{0.978} & 0.964&   0.965 & 0.865 & 0.965\\ 
ArabicDigits & 0.945 & \textbf{0.977} & 0.962 & 0.965 & 0.966 \\
CMU    &\textbf{1.000} & \textbf{1.000} & \textbf{1.000} & \textbf{1.000} & \textbf{1.000} \\
PEMS    & \textbf{0.878} & 0.798 & 0.775 & 0.763 & 0.763 \\
\cmidrule[1.5pt]{1-6}
\end{tabular}
\end{SCtable*}


\paragraph{Classification with missing data}
We used the \textit{Japanese vowels} and \textit{uWave} datasets to illustrate the TCKs ability to classify real-world MTS with missing data. 
We removed different fractions of the values completely at random (MCAR) and ran a 1NN-classifier equipped with TCK, LPS, DTW (i) and GAK.
We also compared to TCK and LPS with imputation of the mean. 
Mean classification accuracies and standard deviations, evaluated over 10 runs, are reported in Fig.~\ref{fig:missing_tck}.

\begin{SCfigure}[0.3][!t] 
\centering
  \includegraphics[width=0.8\textwidth]{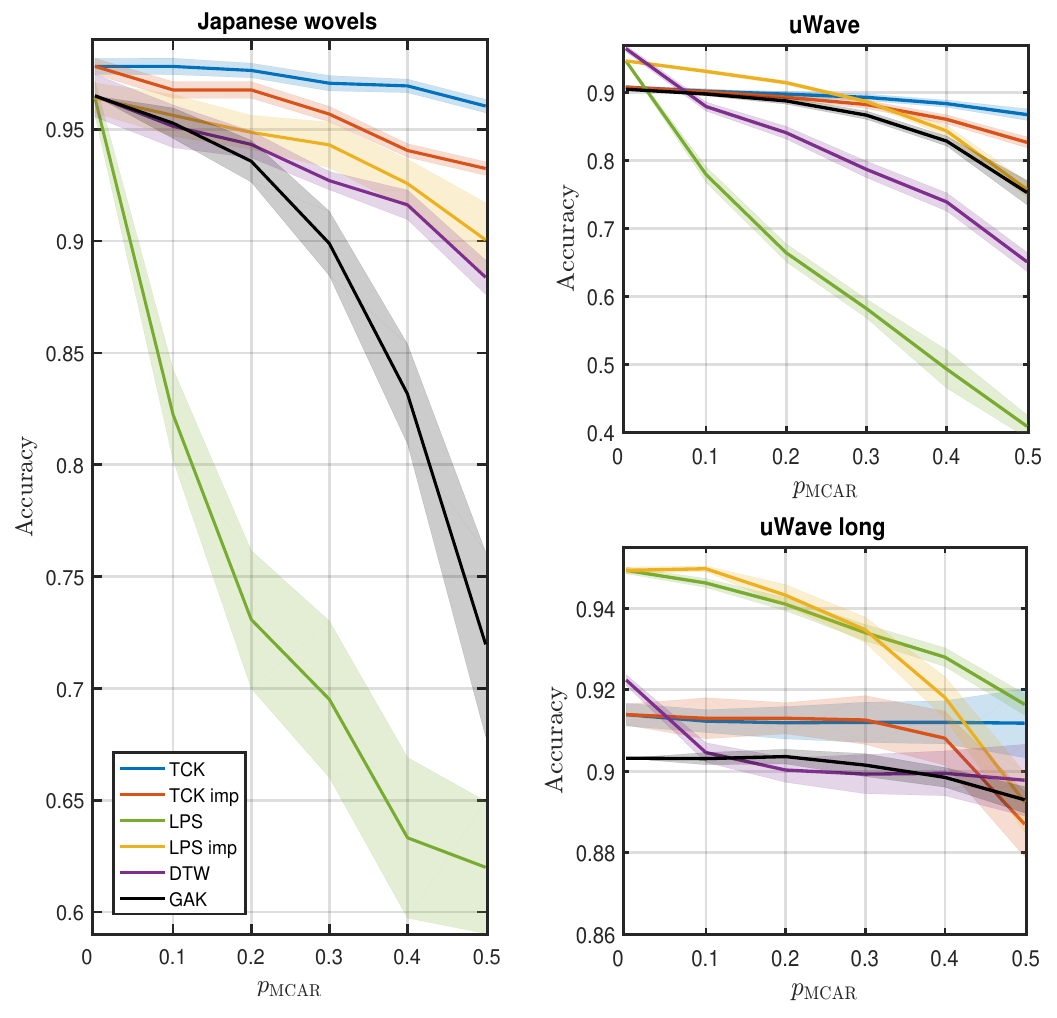}
\caption{Classification accuracies with different proportions of MCAR data for \textit{Japanese vowels} and \textit{uWave}. \textit{uWave long} represent the uWave dataset where the MTS have their original length ($T=315$).   Shaded areas represent standard deviations calculated over 10 independent runs.}
\label{fig:missing_tck}
\end{SCfigure}

On the Japanese vowels dataset the accuracy obtained with LPS decreases very fast as the fraction of missing data increases and is greatly outperformed by LPS imp. The performance of GAK also diminishes quickly. 
The accuracy obtained with DTW (i) decreases from 0.965 to 0.884, whereas TCK imp decreases from 0.978 to 0.932. 
The most stable results are obtained using TCK: as the ratio of missing data increases from 0 to 0.5, the accuracy decreases from 0.978 to 0.960. 
We notice that, even if TCK imp yields the second best results, it is clearly outperformed by TCK. 

Also for the uWave dataset the accuracy decreases rapidly for LPS, DTW and GAK. 
The accuracy for TCK is 0.908 for no missing data, is almost stable up to 30\% missing data and decreases to 0.868 for 50\% missing data. 
TCK imp is outperformed by TCK, especially beyond 20\% missingness. 
We notice that LPS imp gives better results than LPS also for this dataset. 
For ratios of missing data above 0.2 TCK gives better results than LPS imp, even though in absence of missingness the accuracy for LPS is 0.946, whereas TCK yields 0.908 only.

To investigate how TCK works for longer MTS, we classified the uWave dataset with MTS of original length, 315. 
In this case the LPS performs better than for the shorter MTS, as the accuracy decreases from 0.949 to 0.916. 
We also see that the accuracy decreases faster for LPS imp. 
For the TCK the accuracy increased from 0.908, obtained on uWave with MTS of length 25, to 0.914 on this dataset.
TCK still gives a lower accuracy than LPS when there is no missing data. However, we see that TCK is very robust to missing data, since the accuracy only decreases to 0.912 when the missing ratio increases to 0.5. 
TCK imp performs equally well up to 30\% missing data, but performs poorly for higher missing ratios.

These results indicate that, in contrast to LPS, TCK is not sensitive to the length of the MTS. It can deal equally well with short MTS and long MTS.

\paragraph{Dimensionality reduction and visualization}

\begin{figure*}[!t]
\centering
\includegraphics[width=0.95\columnwidth, keepaspectratio=true]{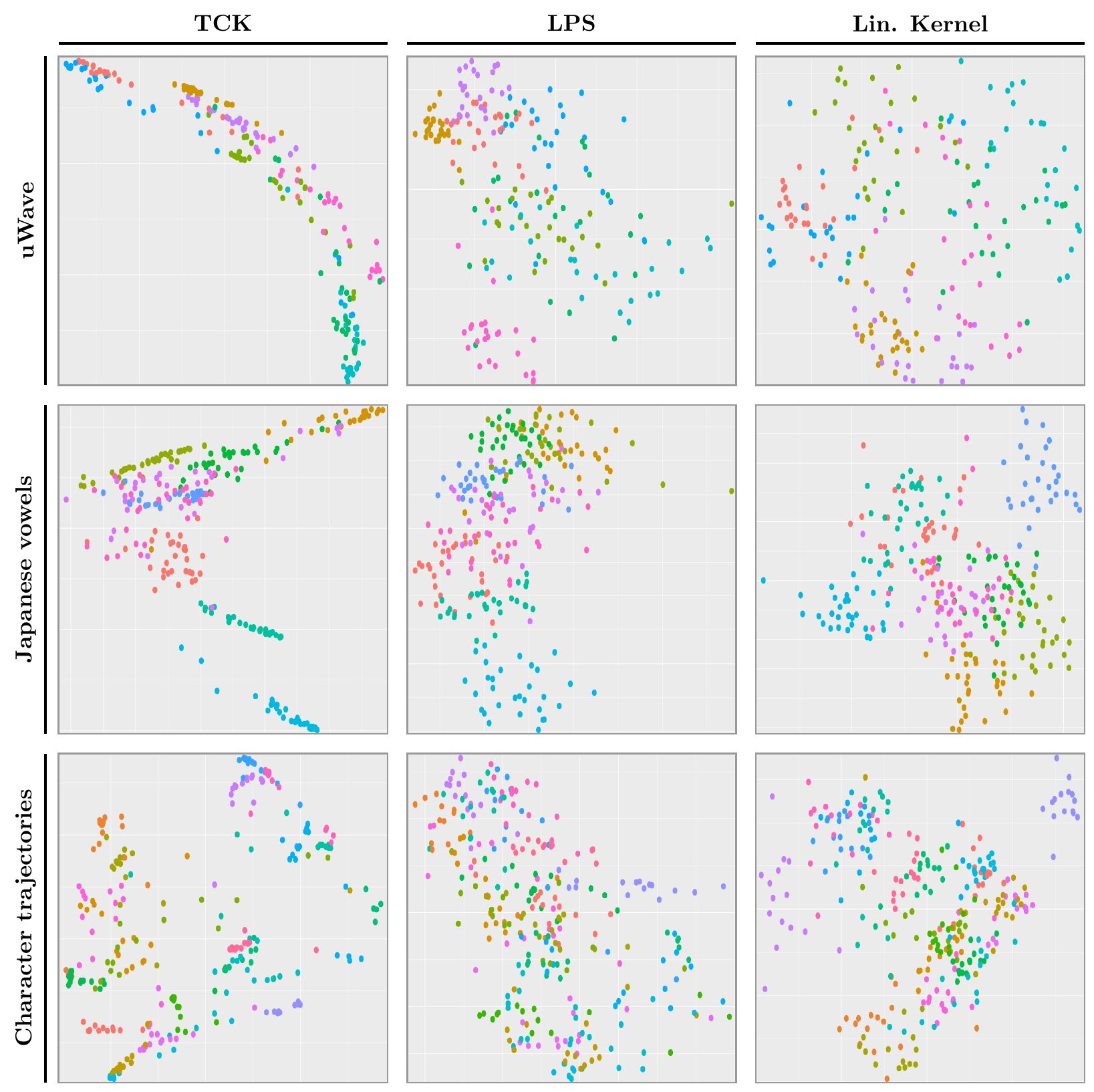}
\caption{Projection of three MTS datasets onto the two top principal components when different kernels are applied. The different colors indicate true class labels.}
\label{fig:dimensionality_reduction_benchmark}
\end{figure*}

In Fig.~\ref{fig:dimensionality_reduction_benchmark} we have plotted the two principal components of \textit{uWave}, \textit{Japanese vowels} and \textit{Character trajectory}, obtained with kPCA configured with TCK, LPS and a linear kernel. 
We notice a tendency in LPS and linear kernel to produce blob-structures, whereas the TCK creates more compact and separated embeddings. 
For example, for Japanese vowels TCK is able to isolate two classes from the rest. %

\subsection{Sensitivity analysis}
The hyperparameters in the TCK are: maximum number of mixtures $C$, number of randomizations $Q$, segment length, subsample size $\mathcal{\eta}$, number of attributes, hyperparameters $\Omega$ and initialization of GMM parameters $\Theta $. 
However, all of them except $C$ and $Q$, are chosen randomly for each $q \in \mathcal{Q}$.  
Hence, the only hyperparameters that have to be set by the user are $C$ and $Q$.

\begin{figure*}[!t]
\centering
  \subfigure
  {
  \includegraphics[width=0.45\textwidth, trim={0mm 1mm 0mm 2mm},clip]{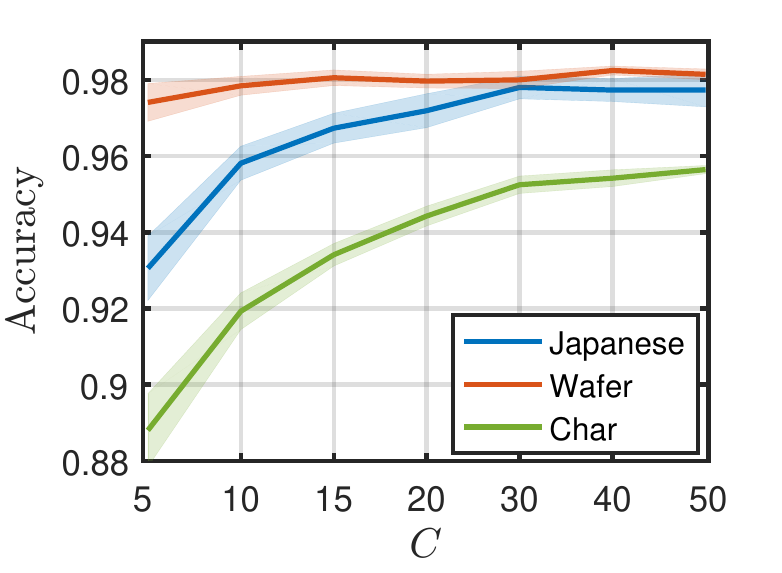}
  \label{fig:sensitivity_C}}\hspace{0em}%
  ~
  \subfigure
  {
  \includegraphics[width=0.45\textwidth, trim={0mm 1mm 0mm 2mm},clip]{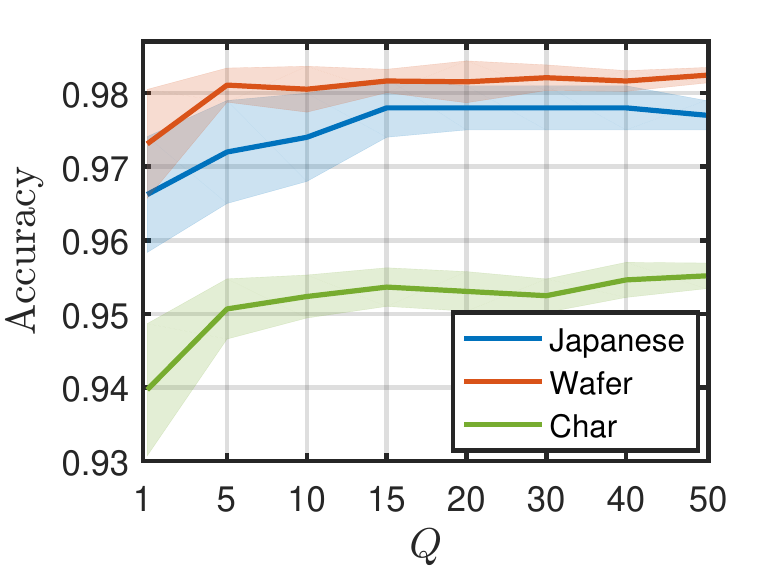}
  \label{fig:sensitivity_Q}}\hspace{0em}%
\caption{Accuracies for (left) $Q = 30$ and varying $C$, and (right) $C = 40$ and varying $Q$, over three datasets. Shaded areas represent standard deviations calculated over 10 replications.}
\label{fig:sensitivity}
\end{figure*}

We have already argued that the method is robust and not sensitive to the choice of these hyperparameters. 
Here, we evaluate empirically TCK's dependency on the chosen maximum number of mixture components $C$ and of randomizations $Q$, on the three datasets \textit{Japanese vowels}, \textit{Wafer} and \textit{Character trajectories}. Fig.~\ref{fig:sensitivity} (left) shows the classification accuracies obtained using TCK in combination with a 1NN-classifier on the three datasets by fixing $Q = 30$ and varying $C$ from 5 to 50. 
We see that the accuracies are very stable for $C$ larger than 15-20. Even for $C = 10$, the accuracies are not much lower. 
Next, we fixed $C= 40$ and varied $Q$ from 5 to 50. Fig.~\ref{fig:sensitivity} (right) shows that the accuracies increase rapidly from $Q=1$, but also that the it stabilizes quite quickly. 
It appears sufficient to choose $Q > 10$, even if the standard errors are a bit higher for lower $Q$.
These results indicate that it is not critical to tune the hyperparameters $C$ and $Q$ correctly, which is important if the TCK should be learned in an unsupervised way.  

\subsection{Computational time}
All experiments were run using an Ubuntu 14.04 64-bit system with 64 GB RAM and an Intel Xeon E5-2630 v3 processor. We used the low-dimensional \emph{uWave}  and the high-dimensional \emph{PEMS} dataset to empirically test the running time of the TCK. To investigate how the running time is affected by the length and number of variables of the MTS, for the PEMS dataset we selected $V=\{963,100,10,2\}$attributes, while for the uWave dataset we let $T=\{315,200,100,25\}$. Tab.~\ref{tab: running time} shows the running times (seconds) for TCK, LPS, GAK and DTW (i) on these datasets. We observe that the TCK is competitive to the other methods and, in particular, that its running time is not that sensitive to increased length or number of attributes.

\begin{SCtable*}[0.5][!t]
\small
\centering
\caption{Running times (seconds) for computing the similarity between the test and training set for two datasets. The time in brackets represents time used to train the models for the methods that need training. For the PEMS dataset we used the original 963 attributes, but also ran the models on subsets consisting of 100, 10 and 2 attributes, respectively. For the uWave dataset we varied the length from $T = 315$ to $T=25$.}
\label{tab: running time}
\begin{tabular}{lllll}
\cmidrule[1.5pt]{1-5}
\textbf{PEMS} &  $V=963$  & $V=100$  & $V=10$ & $V=2$  \\
\cmidrule[.5pt]{1-5}
TCK & 3.6 (116) & 3.5 (115) & 2.5 (84) & 1.2 (31)\\
LPS & 22 (269) & 3.3 (33) & 1.3 (4.5) & 0.9 (2.9) \\
GAK &  514 &  52 &  5.8 & 1.6 \\
DTW (i) &  1031 & 119 & 13 & 3.5\\
\cmidrule[1.5pt]{1-5}
\textbf{uWave} &  $T=315$  & $T=200$  & $T=100$ & $T=25$  \\
\cmidrule[.5pt]{1-5}
TCK & 42 (46) & 39 (45) & 41 (46) & 27 (35)\\
LPS & 26 (17) & 17 (11) & 11 (7)& 6.6 (2.5) \\
GAK & 28  &  25 & 21 & 20 \\
DTW (i) & 506  & 244  & 110  & 59 \\
\cmidrule[1.5pt]{1-5}
\end{tabular}
\end{SCtable*}

\section{Conclusions}
\label{sec:conclusions}
We have proposed a novel similarity measure and kernel for multivariate time series with missing data. The robust time series cluster kernel was designed by applying an ensemble strategy to probabilistic models. TCK can be used as input in many different learning algorithms, in particular in kernel methods.

The experimental results demonstrated that the TCK (1) is robust to hyperparameter settings, (2) is competitive to established methods on prediction tasks without missing data and (3) is better than established methods on prediction tasks with missing data.

In future works we plan to investigate whether the use of more general covariance structures in the GMM, or the use of HMMs as base probabilistic models, could improve TCK.

\section*{Conflict of interest}
The authors have no conflict of interest related to this work.

\section*{Acknowledgement}

This work (Robert Jenssen and Filippo Bianchi) is partially supported by the Research Council of Norway over FRIPRO grant no. 234498 on developing the \emph{Next Generation Learning Machines}. Cristina Soguero-Ruiz is supported by FPU grant AP2012-4225 from Spanish Government.

The authors would like to thank Sigurd Løkse and Michael Kampffmeyer for useful discussions, and Jonas Nordhaug Myhre for proof reading the article.

\appendix
\section{}
\label{app:lps}
\begin{theorem}\label{theorem:LPS}
  LPS is a kernel.
\end{theorem}

\begin{proof}
  The LPS similarity between two time series $X^{(n)}$ and $X^{(m)}$ is computed from the LPS representation, given by the frequency vectors $H(X^{(n)})$ and $H(X^{(m)})$, where $H(X^{(n)}) = \left[ h_{1,1}^{(n)}, \dots, h_{R,J}^{(n)} \right] \in \mathbb{N}_0^{R J}$ being $h_{r,j}^{(n)} \in \mathbb{N}_{0}$ the number of segments of $X^{(n)}$ contained in the leaf $r$ of tree $j$ and $J$ the number of trees \cite{baydogan2016time}.  Let $N_s=T-L-1$ be the total number of segments of length $L$  in the MTS $X$ of length $T$. Without loss of generality we assume that $N_s$ and $R$, the total number of leaves, are constant in all trees. The LPS similarity reads
  \begin{equation}
  \label{eq:lps}
    S\left( X^{(n)}, X^{(m)} \right) = \frac{1}{RJ} \sum \limits_{r=1}^{R} \sum \limits_{j=1}^{J} \min \left( h_{r,j}^{(n)}, h_{r,j}^{(m)} \right) \in [0,1].
  \end{equation}
  We notice that,  if we ignore the normalizing factor, Eq. \ref{eq:lps} is the computation of the intersection between $H(X^{(n)})$ and $H(X^{(m)})$. 
  In order to complete the proof, we now introduce an equivalent binary representation of the frequency vectors in the leaves.
  We represent the leaf $r$ of the tree $j$ as a binary sequence, with $h_{r,j}$ 1s in front and 0s $N_s-h_{r,j}$ in the remaining positions 
  \begin{equation*}
    \bar{H}(X) = \left[ \UOLunderbrace{\UOLoverbrace{1,\dots,1}^{h_{1,1}}, \UOLoverbrace{0,\dots,0}^{N_s-h_{1,1}} }_{\text{leaf} \, (1,1)}, \dots, \UOLunderbrace{\UOLoverbrace{1,\dots,1}^{h_{r,j}}, \UOLoverbrace{0,\dots,0}^{N_s-h_{r,j}} }_{\text{leaf} \, (r,j)}, \dots, \UOLunderbrace{\UOLoverbrace{1,\dots,1}^{h_{R,J}}, \UOLoverbrace{0,\dots,0}^{N_s-h_{R,J}} }_{\text{leaf} \, (R,J)} \right] \in \{0,1\}^{N_sRJ}.
  \end{equation*}
   
  The intersection between $H(X^{(n)})$ and $H(X^{(m)})$, yielded by Eq. \ref{eq:lps}, can be expressed as a bitwise operation through dot product
  \begin{equation}
  \label{eq:bitwise_intersec}
    \left( H(X^{(n)}) \wedge H(X^{(m)}) \right) = \bar{H}(X^{(n)})^T \bar{H}(X^{(m)}),
  \end{equation}
  which is a linear kernel in the linear span of the LPS representations, which is isometric to $\mathbb{R}^{N_sRJ}$. 
\end{proof}

\bibliographystyle{elsarticle-num}

\bibliography{references}

\end{document}